\documentclass[letterpaper, 10 pt, conference]{ieeeconf}  % Comment this line out if you need a4paper
\IEEEoverridecommandlockouts                              % This command is only needed if 
\usepackage{graphics} % for pdf, bitmapped graphics files
\usepackage{epsfig} % for postscript graphics files
\usepackage{mathptmx} % assumes new font selection scheme installed
\usepackage{times} % assumes new font selection scheme installed
\usepackage{amsmath} % assumes amsmath package installed
\usepackage{amssymb}  % assumes amsmath package installed
\usepackage{psfrag} %modify eps figure
\usepackage{subcaption,tikz} %add subfigure
\usepackage{url}
\usepackage[english]{babel}
\usepackage{algorithm,algorithmic}
\newtheorem{theorem}{Theorem}[section]

\newtheorem{lemma}[theorem]{Lemma}

\newcommand{\CASE}[1]{\STATE \textbf{case} #1\textbf{:} \begin{ALC@g}}
\newcommand{\ENDCASE}{\end{ALC@g}}

\newcommand{\DEFAULT}{\STATE \textbf{default:} \begin{ALC@g}}
\newcommand{\ENDDEFAULT}{\end{ALC@g}}
\newcommand{\DEFAULTLINE}[1]{\STATE \textbf{default:} }

\title{\LARGE \bf
Multi-objective Compositions for Collision-Free Connectivity Maintenance in Teams of Mobile Robots *
}

\author{Li Wang, Aaron D. Ames, and Magnus Egerstedt$^\dagger$% <-this % stops a space
\thanks{*The work by the first and third authors was sponsored by Grant No.
1544332 from the U.S. National Science Foundation, and the work
of the second author was sponsored by Grant No. 1239055 from the U.S. National Science Foundation}% <-this % stops a space
\thanks{$^\dagger$Li Wang and Magnus Egerstedt are with the School of Electrical and Computer Engineering, Aaron D. Ames is with the School of Mechanical Engineering and the School of Electrical and Computer Engineering, Georgia Institute of Technology, Atlanta, GA 30332, USA. Email: {\tt\small \{liwang, magnus, ames\}@gatech.edu} }% <-this % stops a space
}

\begin{document}
\maketitle
\thispagestyle{empty}
\pagestyle{empty}

%%%%%%%%%%%%%%%%%%%%%%%%%%%%%%%%%%%%%%%%%%%%%%%%%%%%%%%%%%%%%%%%%%%%%%%%%%%%%%%%
\begin{abstract}
Compositional barrier functions are proposed in this paper to systematically compose multiple objectives for teams of mobile robots. The objectives are first encoded as barrier functions, and then composed using AND and OR logical operators. The advantage of this approach is that compositional barrier functions can provably guarantee the simultaneous satisfaction of all composed objectives. The compositional barrier functions are applied to the example of ensuring collision avoidance and static/dynamical graph connectivity of teams of mobile robots. The resulting composite safety and connectivity barrier certificates are verified experimentally on a team of four mobile robots.
\end{abstract}

%%%%%%%%%%%%%%%%%%%%%%%%%%%%%%%%%%%%%%%%%%%%%%%%%%%%%%%%%%%%%%%%%%%%%%%%%%%%%%%%
\section{INTRODUCTION} \label{sec:intro}
Multi-robot coordination strategies are often designed to achieve team level collective goals, such as covering areas, forming specified shapes, search and surveillance, see e.g. \cite{cortes2002coverage, lawton2003decentralized, bullo2009distributed, mesbahi2010graph}. As the number of robots and the complexity of the task increases, it becomes increasingly difficult to design one single controller that simultaneously achieves multiple objectives, e.g., forming shapes, collision avoidance and connectivity maintenance. Therefore, there is a need to devise a formal approach that can provably compose multiple objectives for the teams of robots.

Multi-objective controls for multi-agent systems have been extensively studied. The recentered barrier function was used to unify the go-to-goal behavior, collision avoidance, and proximity maintenance \cite{panagou2013multi}; however, it was specifically constructed for go-to-goal task and thus can not be extended to complex situations easily. Research in \cite{thibodeau2004cascaded} tried to achieve multiple objectives, e.g., approaching a target position, avoiding collisions, and keeping the goal within field of view, by designing cascaded filters which remove control commands that violate the objectives or constraints; but this method comes with no provable guarantees. \cite{zavlanos2009hybrid} studied connectivity preserving flocking, and simultaneously achieved alignment, cohesion, separation, and connectivity, which is again a task-specific solution. To enable provably correct and more general objective compositions, the non-negotiable objectives, e.g., collision avoidance and connectivity maintenance, are encoded with compositional barrier functions in this paper. Barrier functions, which were explored in various applications such as robotics \cite{ames2014Dynamic}, safety verification \cite{sloth2012compositional}, and adaptive cruise control \cite{ames2014CBF}, can be used to provably ensure the forward invariance of desired sets \cite{prajna2007framework, tee2009barrier, Xu2015Robustness}. Earlier works on safety barrier certificates for multi-robot system \cite{borrmann2015Swarm, wang2016hetero} encoded multiple objectives by assembling multiple barrier functions. The agents are safe if they satisfy the safety barrier certificates, while the existence of a common solution to multiple barrier functions becomes unclear when the number of objectives increases. This motivates our work of composing multiple barrier functions into a single barrier function, so that the solutions to ensure multiple objectives always exist. 

In this paper, compositional barrier functions are applied to provably ensure collision avoidance and graph connectivity for the coordination control of teams of mobile robots. This is motivated by the fact that many of the multi-agent strategies, such as consensus, flocking, and formation control, implicitly assumes collision avoidance, communication graph connectivity, or both \cite{zavlanos2011graph}. These safety and connectivity objectives are often ensured by some secondary controllers, which take over and modify the higher level control command when violations occur. Typical methods used in these secondary controllers are artificial potential functions \cite{park2001obstacle}, behavior based approaches \cite{arkin1998behavior}, and edge energy functions \cite{ji2007distributed}. However, when the team of robots are either too concentrated or too scattered, the avoidance behavior becomes dominant with the robots spending most of the time avoiding collisions or losses of connectivity, and the higher level objectives can not be achieved \cite{roumeliotis2000small}. The idea pursued in this paper is to design a secondary controller, utilizing compositional barrier functions, which is minimally invasive to the higher level controller, i.e., the avoidance behavior only takes place when collisions or losses of connectivity are truly imminent. Similar collision avoidance strategies were explored in \cite{borrmann2015Swarm, wang2016hetero, tomlin1998conflict}.

The main contributions of this paper are twofold. Firstly, compositional barrier functions are introduced to enable more general compositions of multiple non-negotiable objectives with provable guarantees. Methods to compose multiple objectives through AND and OR logical operators are developed, and conditions on which objectives are composable are provided. Secondly, composite safety and connectivity barrier certificates are synthesized with compositional barrier functions, which provably guarantees collision avoidance and connectivity for teams of mobile robots that perform general coordination tasks.

The rest of this paper is organized as follows. Section \ref{sec:pbf} briefly revisits the control barrier function, and extends it to the piecewise smooth case, which is essential to enable the barrier function composition in Section \ref{sec:compose}. The compositional barrier functions are then used to synthesize the safety and connectivity barrier certificates, which ensure collision avoidance and connectivity maintenance for teams of mobile robots, in Section \ref{sec:example}. The resulting safety and connectivity barrier certificates are implemented experimentally on a team of four Khepera III robots in Section \ref{sec:exp}. Conclusions and discussion of future work are the topics of Section \ref{sec:conclude}.

\section{Piecewise Smooth Control Barrier Functions} \label{sec:pbf}
Control barrier functions are a class of Lyapunov-like functions, which can provably guarantee the forward invariance of desired sets without explicitly computing the system's forward reachable sets. This paper follows the idea of a type of barrier functions similar to \cite{ames2014CBF, Xu2015Robustness}, which expands the admissible control space and enables less restrictive controls. In order to encode more general objectives, we will introduce methods to compose barrier functions with AND and OR logical operators in Section \ref{sec:compose}. After composition, these originally smooth barrier functions might become piecewise smooth. Therefore, this section will set the stage for multi-objective composition by constructing Piecewise Barrier Functions (PBF).

Some useful mathematical definitions and tools, i.e., $PC^r-$functions and $B$-derivative, for dealing with piecewise smooth functions are first revisited.

\textit{Definition 2.1}: A continuous function $f:\mathcal{D}\to\mathbb{R}^m$ defined on an open set $\mathcal{D}\subseteq \mathbb{R}^n$ is a $PC^r-$function, $r\geq 1$, if there exists an open neighborhood $V\subseteq \mathcal{D}$ and a finite collection of $C^r$ functions $\{f_1,f_2,...,f_k\}$ at $\forall x_0 \in \mathcal{D}$, such that the index set $I(x_0)=\{i~|~f(x_0)=f_i(x_0),\forall x\in V\}$ is non-empty.

Note that a $PC^r-$function can be viewed as a continuous selection of a finite number of $C^r$ functions on $\mathcal{D}$. The summation, product, superposition, pointwise maximum or minimum operations on $PC^r-$functions still generate $PC^r-$functions \cite{scholtes2012introduction}. $PC^r-$functions have the favourable properties of locally Lipschitz continuous and B-differentiable \cite{scholtes2012introduction}.

\textit{Definition 2.2}: A locally Lipschitz function $f:\mathcal{D} \to\mathbb{R}^m$ defined on an open set $\mathcal{D}\subseteq \mathbb{R}^n$ is B-differentiable at $x_0\in \mathcal{D}$, if its B-derivative $f'(x_0;\cdot):\mathbb{R}^n\to \mathbb{R}^m$ at $x_0$ is well defined, i.e. the limit
\begin{equation}
f'(x_0;q) = \lim_{a\to 0^+} \frac{f(x_0+aq)-f(x_0)}{a},
\end{equation}
in any direction $q\in \mathbb{R}^n$ exists.

For the generality of discussion, consider a dynamical system in control affine form
\begin{equation}
\dot{x}=f(x)+g(x) u,
\label{eqn:system}
\end{equation}
where $x\in \mathbb{R}^n, u \in \mathbb{R}^m$, $f$ and $g$ are locally Lipschitz. (\ref{eqn:system}) is assumed to be forward complete, i.e., solutions $x(t)$ are well defined $\forall t\geq 0$.

Let a set $\mathcal{C}\subseteq\mathcal{D}$ be defined such that
\begin{equation}
\begin{aligned}
\mathcal{C} &= \{x\in\mathbb{R}^n\, | \ B(x)> 0\},\\
\mathcal{C}^C &= \{x\in\mathbb{R}^n\, | \ B(x)=0\},
\end{aligned}
\label{eqn:setdef}
\end{equation}
where the $PC^r-$function $B:\mathcal{D}\to \mathbb{R}$ is constructed to be positive in $\mathcal{C}$ and zero outside of $\mathcal{C}$. This construction of $\mathcal{C}$ and $B(x)$ enables easy compositions of multiple barrier functions, which will become clear in Section \ref{sec:compose}.

\textit{Definition 2.3}: Given a dynamical system defined in (\ref{eqn:system}) and a set $\mathcal{C}\subseteq \mathcal{D}$ defined in (\ref{eqn:setdef}), the $PC^r-$function $B:\mathcal{D}\to \mathbb{R}$ is a Piecewise Barrier Function (PBF) if there exists a class $\mathcal{K}$ function $\alpha$ such that
\begin{equation} \label{eqn:pbfinf}
\sup_{u \in U} [-B'(x;-f(x)-g(x)u) + \alpha(B(x))]\geq0,
\end{equation}
for all $x\in \mathcal{C}$.

Note that $B'(x;-f(x)-g(x)u)$ is the B-derivative of $B(x)$ at $x$ in the direction of $-f(x)-g(x)u$. When $B(x)$ is smooth, it is equivalent to say
$$
-B'(x;-f(x)-g(x)u) =L_f B(x) +L_gB(x) u,
$$
where the Lie derivative formulation comes from
\begin{equation*}
\dot{B}(x) =\frac{\partial B(x)}{\partial x}(f(x)+g(x)u) = L_fB(x)+L_gB(x)u.
\end{equation*}

The B-derivative can be calculated for $PC^r-$functions in a straight forward fashion. Let $\{b_{1}(x), b_{2}(x), ..., b_{k}(x)\}$ be the set of selection functions for $B(x)$, then the B-derivative of $B(x)$ along the direction $q$ is a continuous selection of $\{\nabla b_{1}(x)q, \nabla b_{2}(x)q, ..., \nabla b_{k}(x)q\}$. The B-derivative of $B(x)$ can be determined by selecting the correct directional derivative from this selection set at $x$.

With the definition of PBFs, the admissible control space for the control system is
\begin{equation}
K(x) = \{u\in U ~|~ -B'(x;-f(x)-g(x)u) + \alpha(B(x)) \geq 0 \}
\end{equation}

\begin{theorem} \label{thm:pbf}
{\it Given a set $\mathcal{C}\subseteq \mathcal{D}$ defined by (\ref{eqn:setdef}) with the associated PBF $B:\mathcal{D}\to \mathbb{R}$, any Lipschitz continuous controller $u(x)\in K(x)$ for the dynamical system (\ref{eqn:system}) render $\mathcal{C}$ forward invariant.}
\end{theorem} 
\begin{proof}
If the controller satisfies $u(x)\in K(x)$, then $-B'(x;-f(x)-g(x)u) \geq -\alpha(B(x))$. Apply the chain rule for B-derivative \cite{kuntz1995qualitative}, it can be shown that
\begin{eqnarray*}
\partial_- B(x(t)) &=& -(B\circ x)'(t;-1) \\
&=& -B'(x(t);x'(t;-1)) \\
&=&-B'(x(t);-f(x)-g(x)u),
\end{eqnarray*}
where $\partial_- B(x(t)) = \lim_{a\to t^-} \frac{B(x(t))-B(x(a))}{t-a}$ is the left time derivative of $B(x(t))$. Therefore, $\partial_- B(x(t)) \geq -\alpha(B(x))$.

Consider the differential equation $\dot{z}(t) = - \alpha(z(t))$ with $z(t_0)=B(x(t_0))>0$, its solution is given by
\begin{equation}
z(t) = \sigma(z(t_0),t),\nonumber
\end{equation}
due to Lemma 4.4 of \cite{khalil1996nonlinear}, where $\sigma$ is a class $\mathcal{KL}$ function.

With the Comparison Lemma \cite{khalil1996nonlinear}\footnote{Comparison Lemma also works for functions with left or right differentiability. The proof is similar to \cite{khalil1996nonlinear}, and thus omitted here.}, we can get
\begin{equation}
B(x(t)) \geq \sigma(z(t_0),t).\nonumber
\end{equation}
Using the properties of class $\mathcal{KL}$ function, it can be shown that $B(x(t))>0, \forall t\geq0$. Thus $\mathcal{C}$ is forward invariant.
\end{proof}

To sum up, we can get set invariance properties similar to \cite{ames2014CBF, Xu2015Robustness} using PBFs.

\section{Composition of Multiple Objectives} \label{sec:compose}
In this section, we will use PBFs developed in Section \ref{sec:pbf} to compose multiple non-negotiable objectives with AND and OR logical operators. Each objective is encoded as a set. The objective is satisfied as long as the states of the system stay within the desired set. Define $\mathcal{C}_i \subseteq \mathcal{D}, i=1,2,$ similar to (\ref{eqn:setdef}),
\begin{equation}
\begin{aligned}
\mathcal{C}_i &= \{x\in\mathbb{R}^n\, | \ B_i(x)> 0\},\\
\mathcal{C}_i^C &= \{x\in\mathbb{R}^n\, | \ B_i(x)=0\},
\end{aligned}
\label{eqn:set12}
\end{equation}
Let $B_\cup=B_1+B_2$ and $B_\cap=B_1B_2$,
\begin{equation}
\begin{aligned}
\mathcal{E} &= \{x\in\mathbb{R}^n\, | \ B_\cup(x)> 0\},\\
\mathcal{F} &= \{x\in\mathbb{R}^n\, | \ B_\cap(x)> 0\}.
\end{aligned}
\label{eqn:setun}
\end{equation}
\begin{lemma} \label{lm:setun}
Given $\mathcal{C}_i,i=1,2$ defined in (\ref{eqn:set12}), $\mathcal{E}$ and $\mathcal{F}$ defined in (\ref{eqn:setun}), $\mathcal{E} = \mathcal{C}_1 \cup\mathcal{C}_2$ and $\mathcal{F} = \mathcal{C}_1\cap\mathcal{C}_2$.
\end{lemma}
\begin{proof}
Pick any elements $x_1\in\mathcal{E}$, $x_2\in\mathcal{F}$, we have
\begin{eqnarray}
B_\cup(x_1)&=&B_1(x_1)+B_2(x_1)>0, \label{eqn:bcup}\\
B_\cap(x_2)&=&B_1(x_2)B_2(x_2)>0. \label{eqn:bcap}
\end{eqnarray} 
From the definition (\ref{eqn:set12}), $B_1(x)$ and  $B_2(x)$ are always non-negative. Thus, (\ref{eqn:bcup}) implies $B_1(x_1)>0$ or $B_2(x_1)>0$, i.e. $x_1\in \mathcal{C}_1\cup\mathcal{C}_2$. (\ref{eqn:bcap}) implies $B_1(x_2)>0$ and $B_2(x_2)>0$, i.e. $x_2\in \mathcal{C}_1\cap\mathcal{C}_2$. This means $\mathcal{E} \subseteq \mathcal{C}_1\cup\mathcal{C}_2$ and $\mathcal{F} \subseteq \mathcal{C}_1\cap\mathcal{C}_2$.

Conversely, we can show that $\mathcal{C}_1\cup\mathcal{C}_2 \subseteq \mathcal{E}$ and $\mathcal{C}_1\cap\mathcal{C}_2 \subseteq \mathcal{F}$. This completes the proof.
\end{proof}

With this result, we can compose two objectives into one set using AND or OR logical operators. The existence of a negation operator is not clear in the current problem setup. Note that \textit{Lemma} \ref{lm:setun} shows that $B_\cup$ and $B_\cap$ are precise PBFs to encode AND or OR logical operators, which allows us to have truly minimal invasive avoidance behaviors in Section \ref{sec:mininvasive}.

Next, we will present the result to formally ensure OR logical operator for two objectives using PBFs.
\vspace{.3cm}
\begin{theorem} \label{thm:pbfu}
{\it Given $\mathcal{C}_i,i=1,2,$ defined in (\ref{eqn:set12}), $\mathcal{E}$ defined in (\ref{eqn:setun}), and a valid PBF $B_\cup$ on $\mathcal{E}$, then any Lipschitz continuous controller $u(x) \in K_\cup(x)$ for the dynamical system (\ref{eqn:system}) render $\mathcal{C}_1 \cup \mathcal{C}_2$ forward invariant, where
\begin{equation}
K_\cup(x) = \{u\in U ~|~ -B_\cup'(x;-f(x)-g(x)u) + \alpha(B_\cup(x)) \geq 0 \}. \nonumber
\end{equation}}
\end{theorem}
\begin{proof}
$B_\cup$ is the summation of two $PC^r-$functions, thus still a $PC^r-$function \cite{scholtes2012introduction}. The B-derivative for $B_\cup$ is well-defined at $\forall x\in \mathcal{E}$. Since $B_1(x)$ and  $B_2(x)$ are always non-negative, $B_\cup$ is also non-negative, i.e., $B_\cup>0$ in $\mathcal{E}$, $B_\cup=0$ outside of $\mathcal{E}$. 

When $u(x) \in K_\cup(x)$, we have $\partial_- B_\cup x(t) \geq - \alpha(B_\cup x)$. Apply \textit{Theorem} \ref{thm:pbf}, $\mathcal{E}$ is forward invariant. Use \textit{Lemma} \ref{lm:setun}, we can get $\mathcal{C}_1 \cup \mathcal{C}_2$ is also forward invariant.
\end{proof}
\vspace{.3cm}
Note that $B_i, i=1,2$ are valid PBFs does not imply $B_\cup$ is a valid PBF. We still need to check if $B_\cup$ is a valid PBF before applying \textit{Theorem} \ref{thm:pbfu}, which means
\begin{equation}
\sup_{u \in U} [-B_\cup'(x;-f(x)-g(x)u) + \alpha(B_\cup(x))]\geq0,\nonumber
\end{equation}
for all $x\in \mathcal{C}_1 \cup \mathcal{C}_2$. This condition guarantees that the admissible control space is strictly non-empty.

An easier but more restrictive condition to check for the composibility is
\begin{equation}
\sup_{u \in U} \min_{i =1,2} [-B_i'(x;-f(x)-g(x)u) + \alpha(B_i(x))]\geq0, \nonumber 
\end{equation}
for all $x\in\mathcal{C}_1 \cup \mathcal{C}_2$, which means there is always a common $u$ to satisfy both PBF constraints.

The result for ensuring AND logical operator for two objectives using PBFs can be derived similarly.
\vspace{.3cm}
\begin{theorem} \label{thm:pbfn}
{\it Given $\mathcal{C}_i,i=1,2,$ defined in (\ref{eqn:set12}), $\mathcal{F}$ defined in (\ref{eqn:setun}), and a valid PBF $B_\cap$ on $\mathcal{F}$, then any Lipschitz continuous controller $u(x) \in K_\cap(x)$ for the dynamical system (\ref{eqn:system}) render $\mathcal{C}_1 \cap \mathcal{C}_2$ forward invariant, where
\begin{equation}
K_\cap(x) = \{u\in U ~|~ -B_\cap'(x;-f(x)-g(x)u) + \alpha(B_\cap(x)) \geq 0 \}. \nonumber
\end{equation}}
\end{theorem}

The proof of this theorem is similar to \textit{Theorem} \ref{thm:pbfu}.
\vspace{.3cm}

Up until now, we have a provably correct method for composing multiple objectives. Conditions have also been provided to check whether the objectives are composable using the AND or OR logical operators. Next, the compositional barrier functions will be applied to safety and connectivity maintenance for teams of mobile robots.

\section{Collision avoidance and Connectivity Maintenance for Teams of Mobile Robots} \label{sec:example}
The design of control algorithms for teams of mobile robots often involves simultaneous fulfilment of multiple objectives, e.g., keeping certain formation, covering areas, avoiding collision, and maintaining connectivity. It is oftentimes a challenging task to synthesize a single controller that achieves all these objectives. In this section, we will use the compositional barrier functions to provably ensure safety (in terms of collision avoidance) and connectivity of teams of mobile robots, while achieving higher level collective behaviors.
\subsection{Composite Safety and Connectivity Barrier Certificates}
Let $\mathcal{M}=\{1,2, ... ,N\}$ be the index set of a team of $N$ mobile robots. The mobile robot $i\in\mathcal{M}$ is modelled with double integrator dynamics given by
\begin{equation}\label{eqn:dint}
     \begin{bmatrix}
       \dot{\mathbf{p}}_i  \\[0.3em]
       \dot{\mathbf{v}}_i  \\[0.3em]
     \end{bmatrix}
     = \begin{bmatrix}
       0 & I_{2\times 2}  \\[0.3em]
       0 & 0
     \end{bmatrix}  \begin{bmatrix}  \mathbf{p}_i  \\[0.3em]
       \mathbf{v}_i  \end{bmatrix}
       +  \begin{bmatrix}  0 \\[0.3em]    I_{2\times 2} \end{bmatrix}  \mathbf{u}_i,
\end{equation}
where $\mathbf{p}_i\in \mathbb{R}^2$, $\mathbf{v}_i\in \mathbb{R}^2$, and $\mathbf{u}_i\in \mathbb{R}^2$ represent the current position, velocity and acceleration control input of robot $i$. The ensemble position, velocity, and acceleration of the team of mobile robots are $\mathbf p\in \mathbb{R}^{2N}$, $\mathbf v\in \mathbb{R}^{2N}$, and $\mathbf u\in \mathbb{R}^{2N}$. $x=(\mathbf{p}, \mathbf{v})$ is denoted as the ensemble state of the multi-robot system. The velocity and acceleration of the robot $i$ are bounded by $\|\mathbf{v}_i\|\leq\beta$, and $\|\mathbf{u}_i\|\leq\alpha$. 

In order to use the composite barrier function to ensure safety and connectivity of the team of mobile robots, a mathematical representation of safety and connectivity is formulated first. Two robots $i$ and $j$ need to always keep a safety distance $D_s$ away from each other to avoid collision, meanwhile stay within a connectivity distance $D_c$ of each other to communicate.

Considering the worst case scenario that the maximum braking force of the robots are applied to avoid collision, a pairwise safety constraint between robots $i$ and $j$ can be written as
\begin{equation}
h_{ij}(x) = 2\sqrt{\alpha(\|\Delta \mathbf{p}_{ij}\|-D_s)} + \frac{\Delta \mathbf{p}_{ij}^T}{\|\Delta \mathbf{p}_{ij}\|}\Delta \mathbf{v}_{ij}> 0. \nonumber
\end{equation}
The detailed derivation of this pairwise safety constraint can be found in \cite{borrmann2015Swarm}. A pairwise safe set $\mathcal{C}_{ij}$ and a PBF candidate $B_{ij}(x)$ are defined as
\begin{eqnarray}\label{eqn:2ndsafe}
\mathcal{C}_{ij} &=& \{x~|~ B_{ij}(x) >0\}, \\
B_{ij}(x)&=&\max\{h_{ij}(x),0\}, \nonumber
\end{eqnarray}

In order to ensure the safety of the team of mobile robots, it is important to guarantee that \textit{all} pairwise collisions between the robots are prevented. Therefore the safe set $\mathcal{C}$ for the team of mobile robots can be written as the intersection of \textit{all} pairwise safe sets.
\begin{equation}
\mathcal{C} =  {\bigcap_{\substack{j \in \mathcal{M}\\ j> i}}\mathcal{C}_{ij}},  \label{eqn:safec}
\end{equation}

With the safe set $\mathcal{C}$, we will formally define what is \textit{safe} for the team of mobile robots.

\textit{Definition 4.1}: The team of $N$ mobile robots with dynamics given in (\ref{eqn:dint}) is \textit{safe}, if the ensemble state $x$ stays in the set $\mathcal{C}$ for all time $t\geq 0$.

Let $\mathcal{G}=(V,E)$ be the required connectivity graph, where $V=\{1,2, ..., N\}$ is the set of $N$ mobile robots, $E$ is the required edge set. The presence of a required edge $(i,j)$ indicates that robots $i$ and $j$ should always stay within a connectivity distance of $D_c$.

Similarly, a pairwise connectivity constraint can be developed by considering the worst case scenario, i.e., the maximum acceleration is applied to avoid exceeding the connectivity distance $D_c$. The pairwise connectivity constraint is given as
\begin{equation}
\bar{h}_{ij}(x)= 2\sqrt{\alpha(D_c - \|\Delta \mathbf{p}_{ij}\|)} - \frac{\Delta \mathbf{p}_{ij}^T}{\|\Delta \mathbf{p}_{ij}\|}\Delta \mathbf{v}_{ij}>0. \nonumber
\end{equation}
The corresponding pairwise connectivity set $\bar{\mathcal{C}}_{ij}$ and PBF candidate are
\begin{eqnarray}\label{eqn:2ndconnect}
\bar{\mathcal{C}}_{ij} &=& \{x~|~ \bar{B}_{ij}(x) >0\}, \\
\bar{B}_{ij}(x)&=&\max\{\bar{h}_{ij}(x),0\}. \nonumber
\end{eqnarray}

In order for the team of mobile robots to stay connected, it is necessary to maintain \textit{all} required edges. Therefore, the connectivity set $\bar{\mathcal{C}}$ for the team of mobile robots can be written as
\begin{equation}
\bar{\mathcal{C}} =  {\bigcap_{(i,j) \in E}\bar{\mathcal{C}}_{ij}}.  \label{eqn:connectc}
\end{equation}

With the connectivity set, we can formally define when the team of mobile robots is connected.

\textit{Definition 4.2}: Given a required connectivity graph $\mathcal{G}$, the team of $N$ mobile robots with dynamics given in (\ref{eqn:dint}) is \textit{connected}, if the ensemble state $x$ stays in the set $\bar{\mathcal{C}}$ for all time $t\geq 0$.

In order for the team of mobile robots to stay \textit{safe} and \textit{connected}, the ensemble state $x$ shall stay within 
\begin{equation}\label{eqn:setall}
\mathcal{T} = \underset{\substack{i,j \in \mathcal{M}\\ j>i}}{\bigcap} \mathcal{C}_{ij}\underset{(i,j)\in E}{\bigcap} \bar{\mathcal{C}}_{ij},
\end{equation}
for all time $t\geq 0$. Since $\mathcal{T}$ is the intersection of multiple sets, the compositional barrier function developed in section $\ref{sec:pbf}$ can be used to ensure the forward invariance of $\mathcal{T}$. The composite PBF for safety and connectivity maintenance is proposed to be
\begin{equation} \label{eqn:bcompose}
B(x) = \prod_{\substack{i,j \in \mathcal{M}\\ j>i}} B_{ij}(x) \prod_{(i,j)\in E} \bar{B}_{ij}(x).
\end{equation}
Before using this composite PBF, we need to check whether $B(x)$ is a valid PBF, which is ensured by the following lemma.

\begin{lemma} \label{lm:validpbf}
The composite barrier function candidate $B(x)$ defined in (\ref{eqn:bcompose}) is a valid PBF, i.e.,
\begin{equation} 
\sup_{u \in U} [-B'(x;-f(x)-g(x)\mathbf u) + \alpha(B(x))]\geq0,
\end{equation}
for all $x\in\mathcal{T}$.
\end{lemma}
\begin{proof}
The composite barrier function candidate $B(x)$ defined on $\mathcal{T}$ is a $C^r$ function. Thus it is equivalent to show that
\begin{equation} \label{eqn:infu}
\sup_{u \in U} [L_fB(x) + L_gB(x)\mathbf u+ \alpha(B(x))]\geq0,
\end{equation}
Note that $B(x), B_{ij}(x)$ and $\bar{B}_{ij}(x)$ are all positive in $\mathcal{T}$. Take the logarithm of $B(x)$ and differentiate using the chain rule, we get
\begin{eqnarray}
\ln(B(x)) &=& \sum_{\substack{i,j \in \mathcal{M}\\ j>i}} \ln(B_{ij}) + \sum_{(i,j)\in E} \ln(\bar{B}_{ij}), \nonumber \\
\frac{\dot{B}}{B} &=& \sum_{\substack{i,j \in \mathcal{M}\\ j>i}} \frac{\dot{B}_{ij}}{B_{ij}} +\sum_{(i,j)\in E} \frac{\dot{\bar{B}}_{ij}}{\bar{B}_{ij}}. \nonumber
\end{eqnarray}
Thus the Lie Derivative along $g$ direction is
\begin{flalign}
&\frac{L_g B}{B}\mathbf u = \sum_{\substack{i,j \in \mathcal{M}\\ j>i}} \frac{L_g B_{ij}}{B_{ij}}\mathbf u + \sum_{\substack{(i,j) \in E}} \frac{L_g \bar{B}_{ij}}{\bar{B}_{ij}}\mathbf u, &&\nonumber \\
&= \sum_{\substack{i,j \in \mathcal{M}\\ j>i}} \frac{\Delta \mathbf p_{ij}}{B_{ij}\|\Delta \mathbf p_{ij}\|} \Delta \mathbf u_{ij} - \sum_{\substack{(i,j) \in E}} \frac{\Delta \mathbf p_{ij}}{\bar{B}_{ij}\|\Delta \mathbf p_{ij}\|}\Delta \mathbf u_{ij}, &&\nonumber \\
&= \sum_{\substack{(i,j) \in E}} \frac{\bar{B}_{ij} - B_{ij}}{B_{ij}\bar{B}_{ij}}\frac{\Delta \mathbf p_{ij}}{\|\Delta \mathbf p_{ij}\|}\Delta \mathbf u_{ij} + \sum_{\substack{(i,j) \notin E}} \frac{\Delta \mathbf p_{ij}}{B_{ij}\|\Delta \mathbf p_{ij}\|}\Delta \mathbf u_{ij},&&\nonumber \\
&= \sum_{i\in \mathcal{M}} \left[ \sum_{\substack{j|(i,j) \in E}} \frac{\bar{B}_{ij} - B_{ij}}{B_{ij}\bar{B}_{ij}}\frac{\Delta \mathbf p_{ij}}{\|\Delta \mathbf p_{ij}\|} + \sum_{\substack{j|(i,j) \notin E}} \frac{\Delta \mathbf p_{ij}}{B_{ij}\|\Delta \mathbf p_{ij}\|} \right] \mathbf{u}_i.&&\nonumber
\end{flalign}
When $L_f B=\mathbf{0}$, we have 
\begin{equation}\label{eqn:lfbui}
\sum_{\substack{j|(i,j) \in E}} \frac{\bar{B}_{ij} - B_{ij}}{B_{ij}\bar{B}_{ij}\|\Delta \mathbf p_{ij}\|}\Delta \mathbf p_{ij} + \sum_{\substack{j|(i,j) \notin E}} \frac{\Delta \mathbf p_{ij}}{B_{ij}\|\Delta \mathbf p_{ij}\|}=0,\forall i\in \mathbf{M}.
\end{equation}
Define a diagonal weight matrix $W=diag(\omega_{ij})\in\mathbb{R}^{\frac{N(N-1)}{2}\times\frac{N(N-1)}{2}}$ for a complete graph, i.e., all vertexes are connected to each other, where
\begin{equation}
\omega_{ij} =
  \begin{cases}
    \frac{\bar{B}_{ij} - B_{ij}}{B_{ij}\bar{B}_{ij}\|\Delta \mathbf p_{ij}\|}      &, \quad \text{if } (i,j)\in E,\\
    \frac{1}{B_{ij}\|\Delta \mathbf p_{ij}\|},  & \quad \text{if } (i,j)\notin E,\\
  \end{cases} \nonumber
\end{equation}
Let $W^{1/2}=diag(\sqrt{\omega_{ij}})$, note $\omega_{ij}$ can be negative, in which case $W^{1/2}$ contains imaginary elements. Denote $D=[D_{ij}]\in\mathbb{R}^{N\times\frac{N(N-1)}{2}}$ as the incidence matrix for a complete graph with random orientations,
\begin{equation}
D_{ij} =
  \begin{cases}
    1      &, \quad \text{if vertex } i \text{ is the tail of edge } j,\\
    -1   &, \quad \text{if vertex } i \text{ is the tail of edge } j.\\
  \end{cases} \nonumber
\end{equation}
Then (\ref{eqn:lfbui}) can be written as
\begin{equation}
D WD^T [\mathbf p_1, \mathbf p_2, ..., \mathbf p_N]^T=0, \nonumber
\end{equation}
which implies $W^{1/2}D^T [\mathbf p_1, \mathbf p_2, ..., \mathbf p_N]^T=0$. 

If $\exists \omega_{ij}\neq 0$, then $\mathbf p_i =\mathbf p_j$. This is impossible, because agents $i$ and $j$ can't be on top of each other in $\mathcal{C}_{ij}$. Therefore, in almost all cases, we have $L_gB\neq 0$. A control action $\mathbf u$ can always be found that shows (\ref{eqn:infu}) is satisfied.

If $\nexists \omega_{ij}\neq 0$, i.e., all weights $\omega_{ij}$ are zero, then the required connectivity graph is a complete graph and $\bar{B}_{ij}=B_{ij}, \forall i\neq j$. It can be shown that $L_fB$ is non-negative in this case. Therefore, in this trivial case, we have $L_gB=0,L_fB>-\alpha(B)$ for any class $\mathcal{K}$ function $\alpha$. Any control action $\mathbf{u}$  can validate that (\ref{eqn:infu}) is satisfied.

To sum up, the composite safety and connectivity barrier function $B(x)$ satisfies (\ref{eqn:infu}) $\forall x\in\mathcal{T}$, and is thus a valid PBF.
\end{proof}

\textit{Lemma} \ref{lm:validpbf} also implies that the admissible control space,
\begin{equation}\label{eqn:ukx}
K_{\mathcal{T}}(x) = \{\mathbf u\in U~|~ L_fB(x)+L_gB(x)\mathbf{u}+\alpha(B(x))\geq 0\},
\end{equation} 
is always non-empty. With this result, we will present the main theorem of this paper.
\vspace{.3cm}
\begin{theorem}\label{thm:pbfcompose}
{\it Given any required connectivity graph $\mathcal{G}=(V,E)$, a PBF $B(x)$ defined in (\ref{eqn:bcompose}), any Lipschitz continuous controller $\mathbf{u}(x) \in K_{\mathcal{T}}(x)$ for the dynamical system (\ref{eqn:dint}) guarantees that the team of mobile robots are \textit{safe} and \textit{connected}.}
\end{theorem}
\begin{proof}
\textit{Lemma} \ref{lm:validpbf} ensures that $B(x)$ is a valid PBF defined for the set $\mathcal{T}$ in (\ref{eqn:setall}). Thus when $\mathbf{u}(x) \in K_{\mathcal{T}}(x)$,  $\mathcal{T}$ is forward invariant from \textit{Theorem} \ref{thm:pbf}, i.e., $B(x)>0, \forall t>0$. From definitions (\ref{eqn:2ndsafe}), (\ref{eqn:2ndconnect}), and (\ref{eqn:bcompose}), all PBFs are constructed to be non-negative. Therefore, 
\begin{eqnarray*}
B_{ij}>0,& \forall i,j \in \mathcal{M}, j>i, &\forall t>0, \\
\bar{B}_{ij}>0,& \forall (i,j)\in E, &\forall t>0.
\end{eqnarray*}
Both $\mathcal{C}$ and $\bar{\mathcal{C}}$ are forward invariant. $\mathcal{C}$ encodes that all agents do not collide with each other, while $\bar{\mathcal{C}}$ encodes that all connectivity requirements specified by the graph $\mathcal{G}$ are satisfied, i.e., the team of mobile robots are \textit{safe} and \textit{connected}.
\end{proof}
\vspace{.3cm}

\textit{Theorem} \ref{thm:pbfcompose} ensures that the team of mobile robots remains \textit{safe} and \textit{connected} as long as the controller $\mathbf{u}(x)$ stays within the admissible control space $K_{\mathcal{T}}(x)$. Up until now, we have a strategy to formally ensure safety and connectivity of the team of mobile robots. Next, an optimization based controller will be presented to inject higher level goals, e.g., visiting waypoints, form certain shapes, and covering area, into the controller design.

\subsection{Minimally Invasive Optimization based Controller}\label{sec:mininvasive}
Designing a single controller for a multi-robot system that achieves certain goals while ensuring safety and connectivity might render the problem untraceable. An alternative approach is to design a nominal controller $\hat{\mathbf{u}}$ that assumes safety and connectivity, and then correct the controller in a minimally invasive way when it violates safety or connectivity. This is achieved by running the following QP-based controller, 
\begin{equation}
\label{eqn:QPcontroller}
 \begin{aligned}
\mathbf{u}^* =  & \:\: \underset{\mathbf u}{\text{argmin}}
 & & J(\mathbf u) = \sum_{i=1}^{N} \|{\mathbf u}_{i} - \hat{\mathbf u}_{i} \|^2    \\ 
 & \quad \text{s.t.}
 & & L_fB(x)+L_gB(x)\mathbf{u}+\alpha(B(x))\geq 0,   \\
 &
 & &    \| \mathbf u_i\|_\infty  \leq \alpha_i,\:\: \forall i \in\mathcal{M}.
 \end{aligned}
\end{equation}
The control barrier constraint (\ref{eqn:QPcontroller}) is also referred to as the composite safety and connectivity barrier certificates. This QP-based controller allows the nominal controller to execute as long as it satisfies the composite safety and connectivity barrier certificates. When violations of safety or connectivity are imminent, the nominal controller will be modified with a minimal possible impact in the least-squares sense. By running this QP-based controller, the higher level objectives specified by the nominal controller are unified with the safety and connectivity requirements encoded by the safety and connectivity barrier certificates.

\subsection{Maintaining Dynamical Connectivity Graphs}
Due to the dynamically changing environment and robot states, it would sometimes be favourable to allow the robots to switch between different connectivity graphs \cite{kok2003multi}. Motivated by the need of maintaining dynamically changing connectivity graphs, composite safety and connectivity barrier certificates are proposed to ensure safety and dynamical connectivity of the team of mobile robots.

Let $\tilde{\mathcal{G}} = \{\mathcal{G}_1,\mathcal{G}_2, ..., \mathcal{G}_M\}$ denote the set of all allowable connectivity graphs, where $\mathcal{G}_i = (V,E_i), i \in \mathcal{P}$, $\mathcal{P}=\{1,2, ... ,M\}$ is the index set of $\tilde{\mathcal{G}}$. To stay connected, the team of mobile robots needs to satisfy at least one of these allowable connectivity graphs. The set that encodes the dynamical connectivity graph requirement is
\begin{equation}
\tilde{\mathcal C} = \bigcup_{k\in\mathcal P} \underset{(i,j)\in E_k}{\bigcap} \bar{\mathcal{C}}_{ij}
\end{equation}

\textit{Definition 4.3}: Given a set of allowable connectivity graphs $\tilde{\mathcal{G}}$, the team of $N$ mobile robots with dynamics given in (\ref{eqn:dint}) is \textit{dynamically connected}, if the ensemble state $x$ stays in the set $\tilde{\mathcal{C}}$ for all time $t\geq 0$.

In order for the team of mobile robots to stay both \textit{safe} and \textit{dynamically connected}, the ensemble state $x$ shall stay in
\begin{equation}\label{eqn:setalldyna}
\tilde{\mathcal{T}} = \left(\underset{\substack{i,j \in \mathcal{M}\\ j>i}}{\bigcap} \mathcal{C}_{ij} \right) \left( \bigcup_{k\in\mathcal P} \underset{(i,j)\in E_k}{\bigcap} \bar{\mathcal{C}}_{ij} \right),
\end{equation}
for all time $t\geq 0$. Safety and dynamical connectivity guarantees similar to \textit{Theorem} \ref{thm:pbfcompose} can be achieved by using a composite PBF introduced in Section \ref{sec:compose},
\begin{equation} \label{eqn:bcomposedyna}
\tilde{B}(x) = \left(\prod_{\substack{i,j \in \mathcal{M}\\ j>i}} B_{ij}(x)\right) \left(\sum_{k\in \mathcal{P}}\prod_{(i,j)\in E_k} \bar{B}_{ij}(x)\right).
\end{equation} 
It can be shown that $\tilde{B}(x)$ is a valid PBF on $\tilde{\mathcal{T}}$ using the same techniques like \textit{Lemma} \ref{lm:validpbf}, i.e., the admissible control space
\begin{equation}\label{eqn:ukx}
K_{\tilde{\mathcal{T}}}(x) = \{\mathbf u\in U~|~ L_f\tilde{B}(x)+L_g\tilde{B}(x)\mathbf{u}+\alpha(\tilde{B}(x))\geq 0\},
\end{equation} 
is always non-empty.

\vspace{.3cm}
\begin{theorem}
{\it Given a set of allowable connectivity graphs $\tilde{\mathcal{G}} = \{\mathcal{G}_1,\mathcal{G}_2, ..., \mathcal{G}_M\}$, a PBF $\tilde{B}(x)$ defined in (\ref{eqn:bcomposedyna}), any Lipschitz continuous controller $\mathbf{u}(x) \in K_{\tilde{\mathcal T}}(x)$ for the dynamical system (\ref{eqn:dint}) guarantees that the team of mobile robots are \textit{safe} and \textit{dynamically connected}.}
\end{theorem}
The proof of this theorem is similar to \textit{Lemma} \ref{lm:validpbf}, \textit{Theorem} \ref{thm:pbfu}, and \textit{Theorem} \ref{thm:pbfcompose}.
\vspace{.3cm}
\section{Robotic Implementations} \label{sec:exp}
The composite safety and connectivity barrier certificates were tested on a team of four Khepera robots. The real-time positions of the robots are tracked by the Optitrack Motion Capture System. The mutli-robot communications and controls are executed on the Robot Operating System (ROS).

The nominal controller was designed as a waypoint controller, which used a go-to-goal behavior to visit the specified waypoints without considering safety and connectivity. As illustrated in Fig. \ref{fig:waypoints}, each robot needs to visit three waypoints sequencially. Those waypoints are intentionally designed to make robots collide at multiple places. 
\begin{figure}[h]
  \centering
  \resizebox{1.8in}{!}{\includegraphics{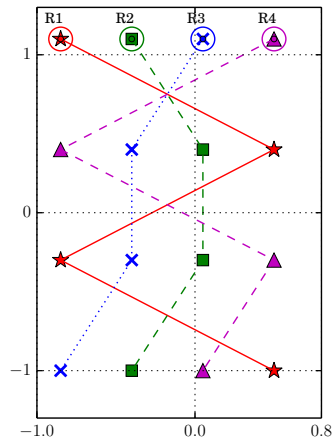}} %1.8
  \caption{Planned waypoints for four robot agents. R$i$ stands for robot $i$, where $i=1,2,3,4$. The lines represent the nominal trajectories of the robots if they execute the nominal waypoint controller.}
  \label{fig:waypoints}
\end{figure}

\subsection{Composite Safety Barrier Certificates}
In the first experiment, the composite safety barrier certificates were wrapped around the nominal waypoint controller using the QP-based strategy (\ref{eqn:QPcontroller}). The composite PBF was formulated as 
\begin{equation*}
B = B_{12}B_{13}B_{14}B_{23}B_{24}B_{34},
\end{equation*}
so that all possible pairwise collisions are avoided. No connectivity constraints were considered in this experiment.

As shown in Fig. \ref{fig:exp1dist}, all the inter-robot distances are always larger than the safety distance $D_s$, i.e., no collision happened during the experiment. Fig. \ref{fig:expsafe} are snapshots taken by an overhead camera and plotted robot trajectories. All robots successfully visited the specified waypoints without colliding into each other. Note that without the connectivity constraints, the mobile robot team sometimes got disconnected  during the experiment, e.g., the team split into two parts in \ref{fig:t1e}.

\begin{figure}[h]
  \centering
  \resizebox{3.2in}{!}{\includegraphics{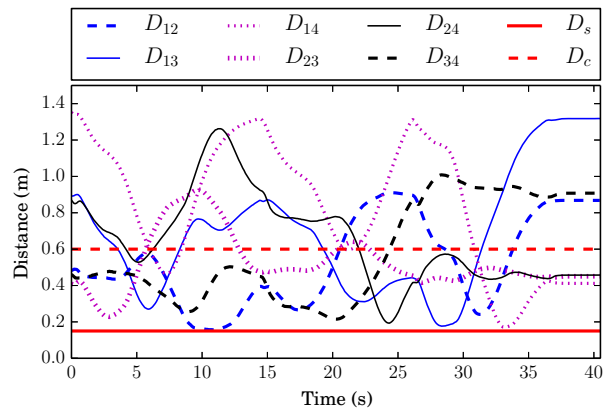}} %3.3
  \caption{Evolution of the inter-robot distances during the experiment. $D_{ij}$ represents the distance between robot $i$ and robot $j$. $D_s=0.15m$ and $D_c=0.6$ are the safety and connectivity distance. $D_{ij}>D_s$ implies that robots $i$ and $j$ did not collide.}
  \label{fig:exp1dist}
\end{figure}

\subsection{Composite Safety and Connectivity Barrier Certificates}
During the second experiment, the composite safety and connectivity barrier certificates were wrapped around the waypoint controller using the QP-based strategy (\ref{eqn:QPcontroller}). The composite PBF is designed as
\begin{equation*}
B = B_{12}B_{13}B_{14}B_{23}B_{24}B_{34}\bar{B}_{23}(\bar{B}_{12}+\bar{B}_{13})(\bar{B}_{24}+\bar{B}_{34}),
\end{equation*}
which encodes that: 1) there should be no inter-robot collisions; 2) robot $2$ and $3$ should always be connected; 3) robot $1$ should be connected to robot $2$ or $3$; 4) robot $4$ should be connected to robot $2$ or $3$.

As shown in Fig. \ref{fig:exp2dist}, the inter-robot distances were always larger than $D_s$, i.e., the team of mobile robots did not collide with each other during the experiment. At the same time, all the connectivity constraints were satisfied, i.e., 1) $D_{23}$ was always smaller than $D_c$; 2) $\min\{D_{12},D_{13}\}$ was always smaller than $D_c$; 2) $\min\{D_{24},D_{34}\}$ was always smaller than $D_c$. The team of mobile robots satisfied all the safety and connectivity requirements specified by the safety and connectivity barrier certificates.
\begin{figure}[t!]
  \centering
  \resizebox{3.2in}{!}{\includegraphics{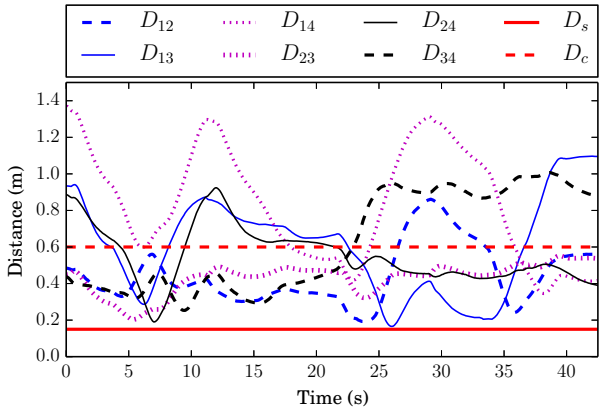}}%3.3
  \caption{Evolution of the inter-robot distances during the experiment. $D_{ij}$ represents the distance between robot $i$ and robot $j$. $D_s=0.15m$ and $D_c=0.6$ are the safety and connectivity distance. $D_{ij}>D_s$ implies that robots $i$ and $j$ do not collide. $D_{ij}<D_c$ implies that robots $i$ and $j$ are in connectivity range. }
  \label{fig:exp2dist}
\end{figure}

The snapshots during the experiment in Fig. \ref{fig:expsc} illustrated that the robots visited all specified waypoints except the last one. This is because the last set of waypoints violated the connectivity constraints, i.e., robot 1 can't reach its waypoint without breaking its connectivity to robot 2 and 3. This experiment also indicates that not all higher level objectives are compatible with the safety and connectivity constraints.

\section{Conclusion and future work} \label{sec:conclude}
This paper presented a systematic way to compose multiple objectives using the compositional barrier functions. AND and OR logical operators were designed to provably compose multiple non-negotiable objectives, with conditions for composibility provided. The composite safety and connectivity barrier certificates were synthesized using the compositional barrier functions to formally ensure safety and connectivity for teams of mobile robots. The resulting barrier certificates were then combined with the general higher level objectives using an optimization-based controller. Robotic experimental implementations validated the effectiveness of the proposed method.

\begin{figure}[H]
\centering
\begin{subfigure}{.21\textwidth}
  \centering
  \resizebox{1.5in}{!}{\includegraphics{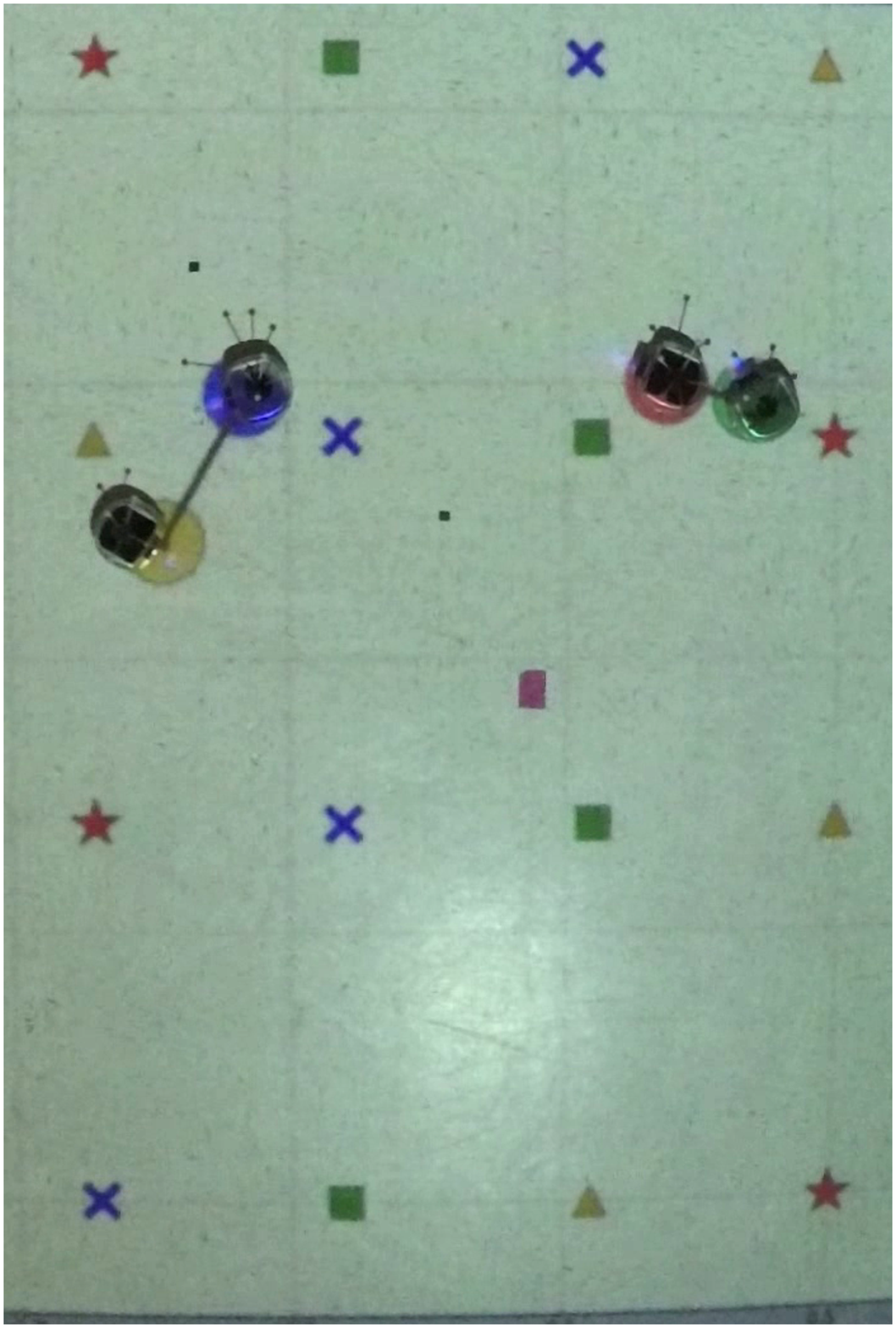}}
  \caption{Agents at 10.0s}
  \label{fig:t1e}
\end{subfigure}%
\begin{subfigure}{.23\textwidth}
  \centering
  \resizebox{1.8in}{!}{\includegraphics{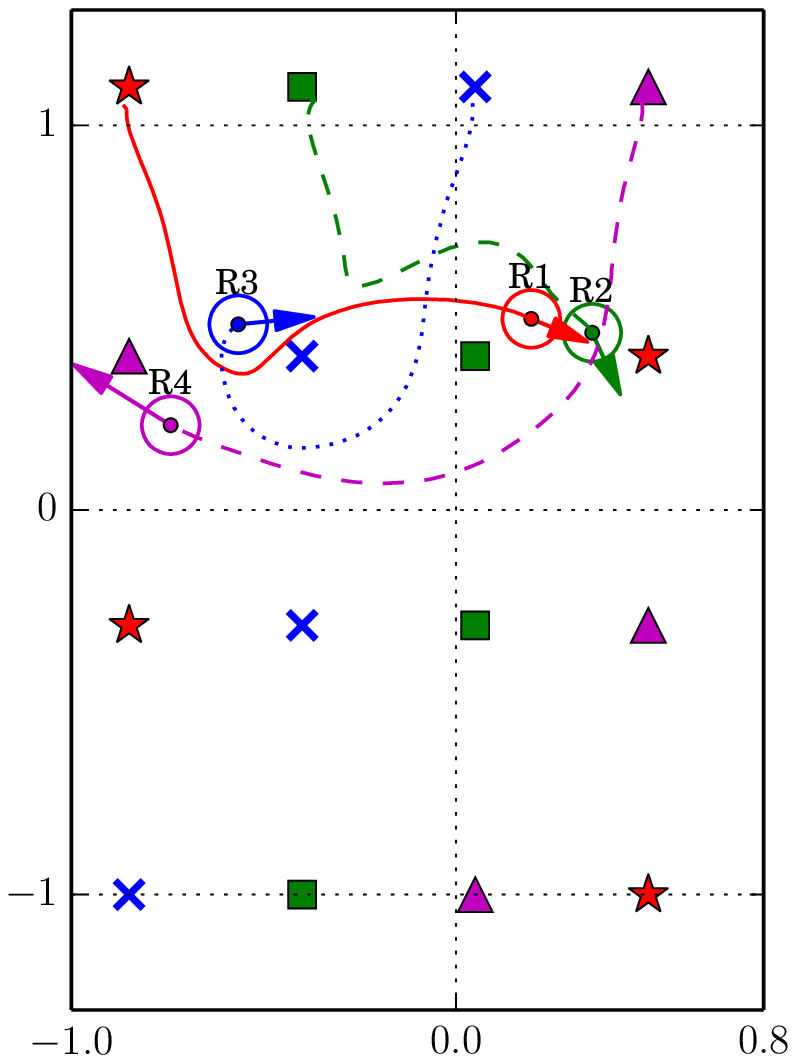}}
  \label{fig:t1p}
\end{subfigure}%
\\
\begin{subfigure}{.21\textwidth}
  \centering
  \resizebox{1.5in}{!}{\includegraphics{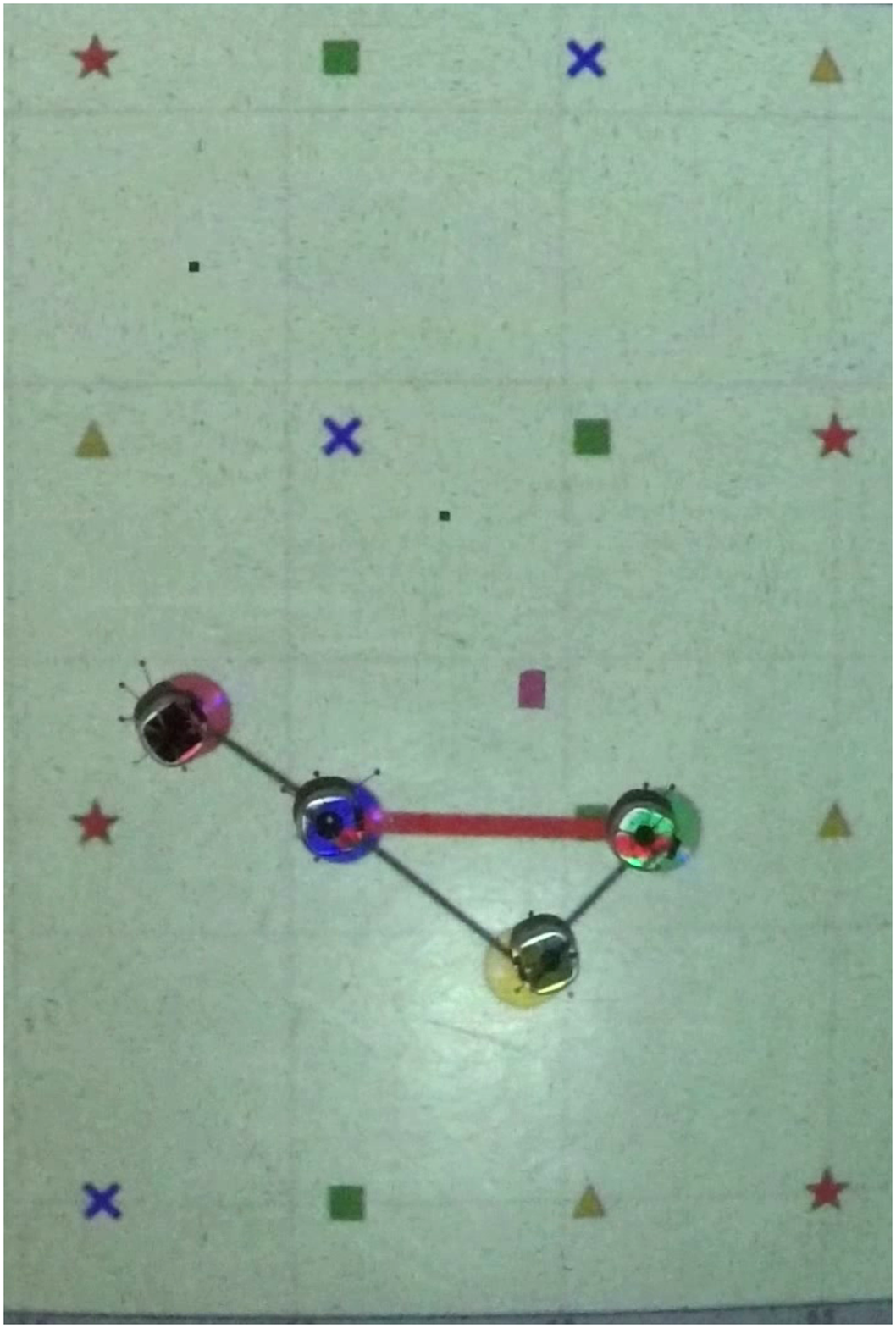}}
  \caption{Agents at 23.0s}
  \label{fig:t2e}
\end{subfigure}
\begin{subfigure}{.23\textwidth}
  \centering
  \resizebox{1.8in}{!}{\includegraphics{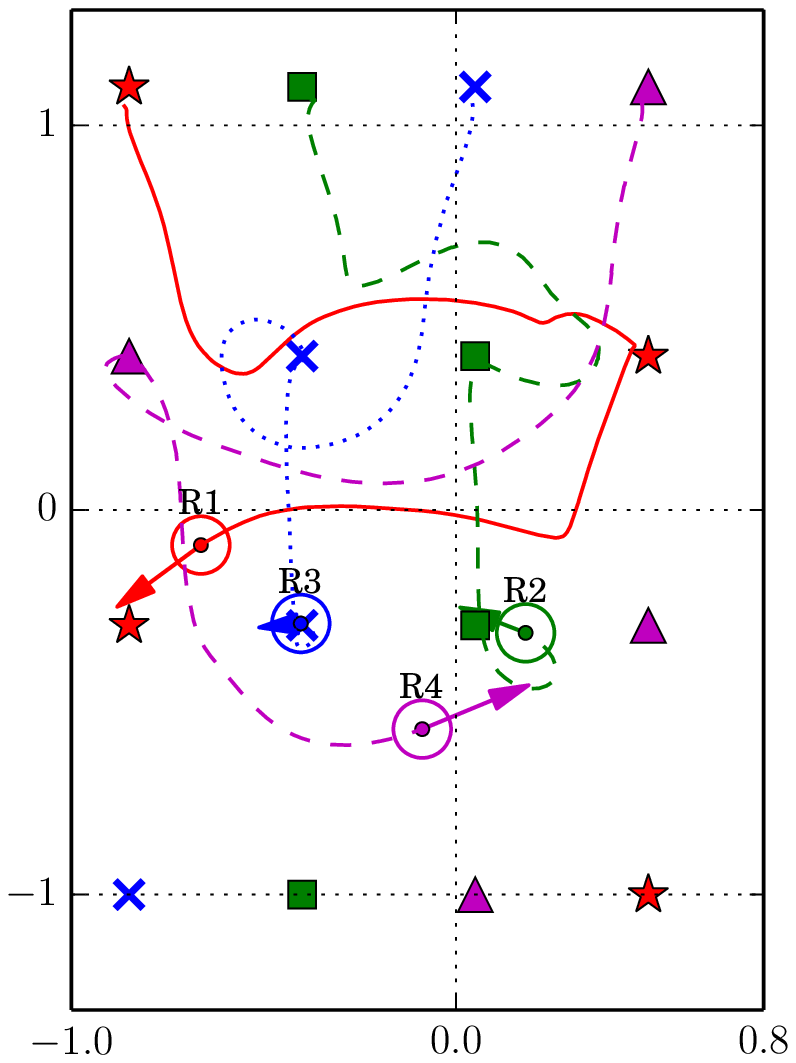}}
  \label{fig:t2p}
\end{subfigure}
\\
\begin{subfigure}{.21\textwidth}
  \centering
  \resizebox{1.5in}{!}{\includegraphics{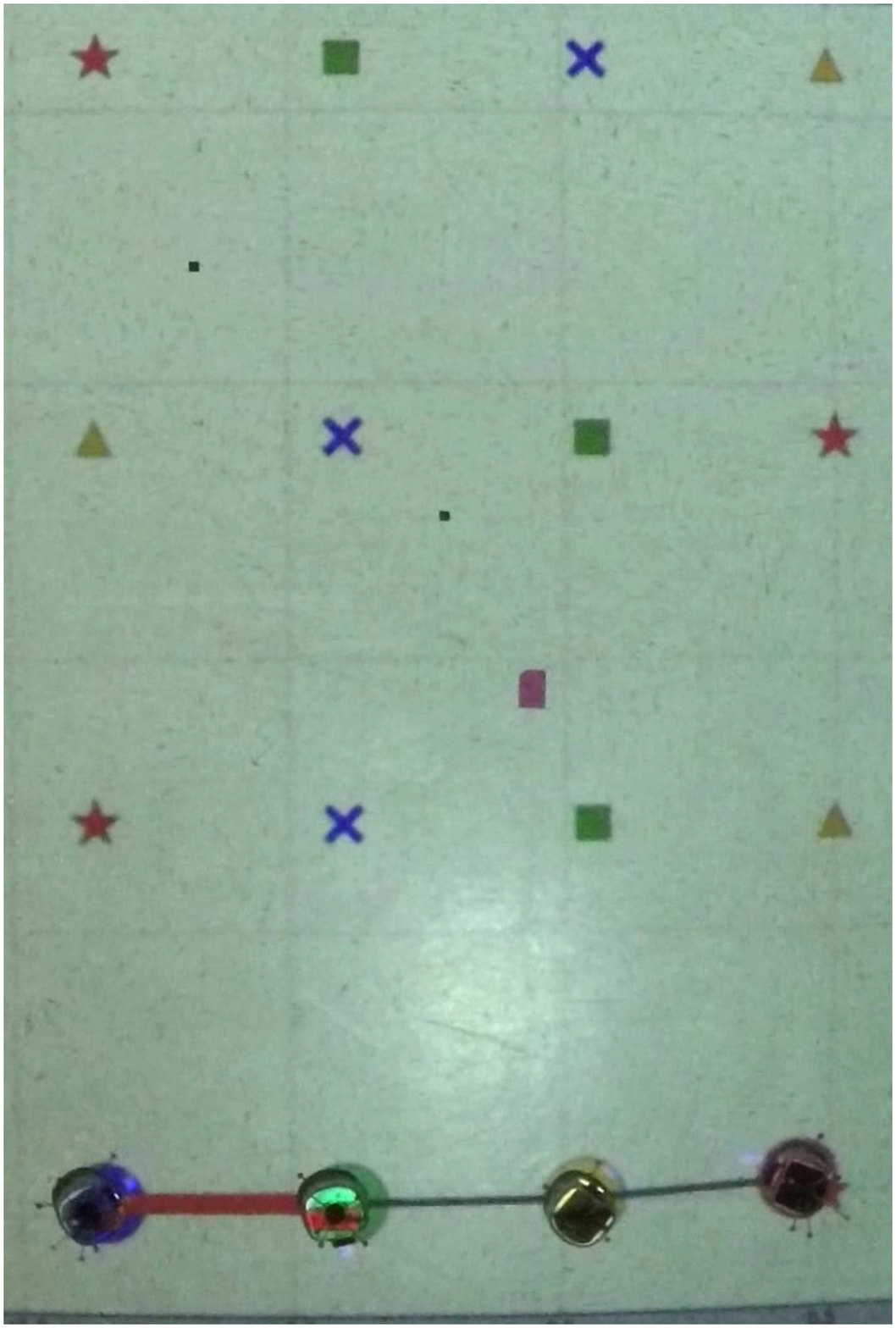}}
  \caption{Agents at 36.0s}
  \label{fig:t3e}
\end{subfigure}%
\begin{subfigure}{.23\textwidth}
  \centering
  \resizebox{1.8in}{!}{\includegraphics{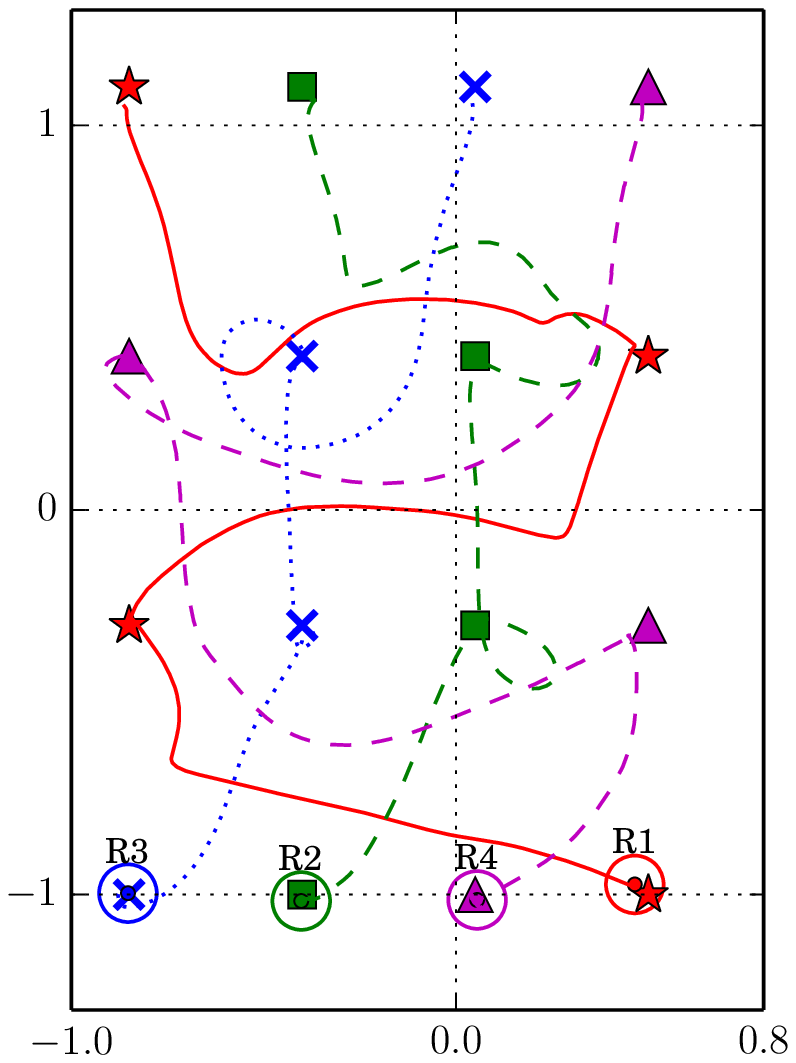}}
  \label{fig:t3p}
\end{subfigure}%
\caption{Experiement of four mobile robots executing waypoint controller regulated by safety barrier certificates. Pictures on the left are taken by an overhead camera. The star, square, cross and triangular markers representing waypoints are projected onto the ground. A straght line connecting two robots were projected onto the ground if the two robots are closer than $D_c=0.6m$. Figures on the left visualize the trajectories, current poisitions and current velocities of the robots. A video of the experiment can be found online \cite{Compose:video}.}\label{fig:expsafe}
\end{figure}

\begin{figure}[H]
\centering
\begin{subfigure}{.21\textwidth}
  \centering
  \resizebox{1.5in}{!}{\includegraphics{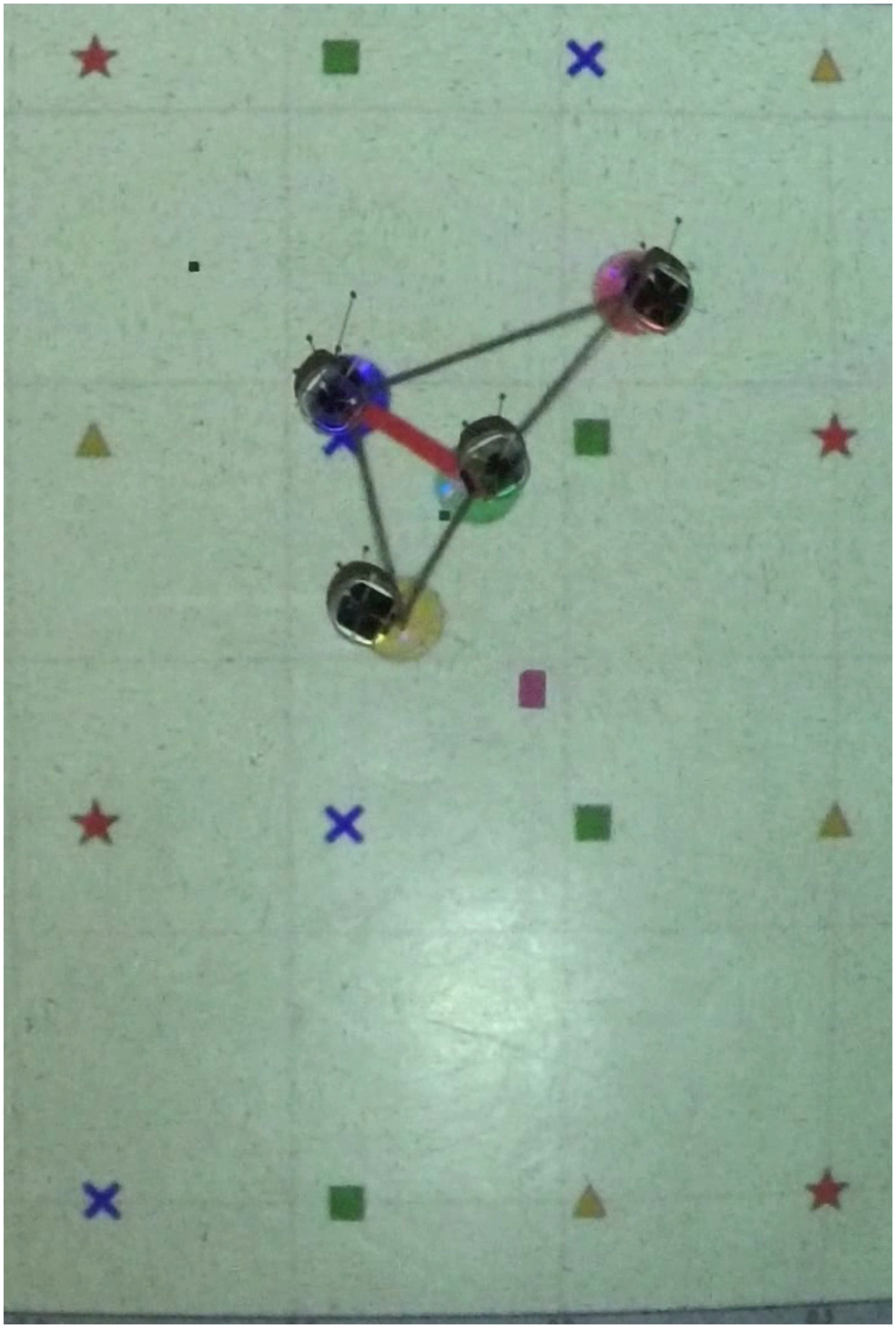}}
  \caption{Agents at 8.0s}
  \label{fig:t1ec}
\end{subfigure}%
\begin{subfigure}{.23\textwidth}
  \centering
  \resizebox{1.8in}{!}{\includegraphics{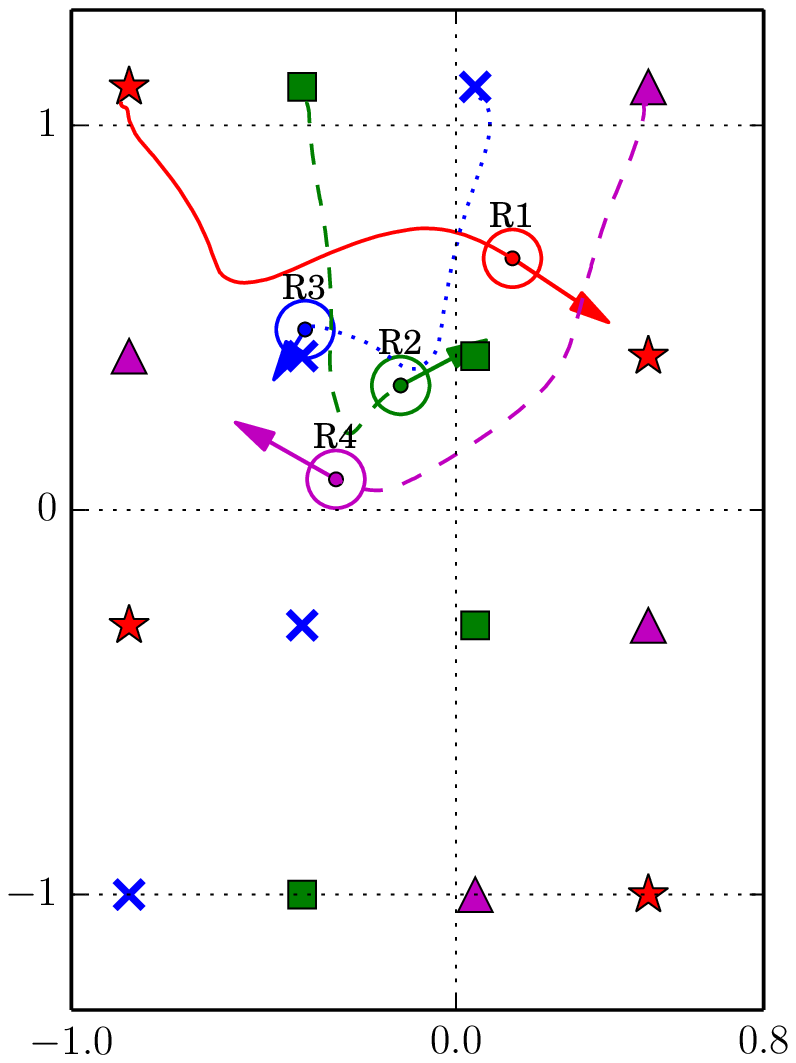}}
  \label{fig:t1pc}
\end{subfigure}%
\\
\begin{subfigure}{.21\textwidth}
  \centering
  \resizebox{1.5in}{!}{\includegraphics{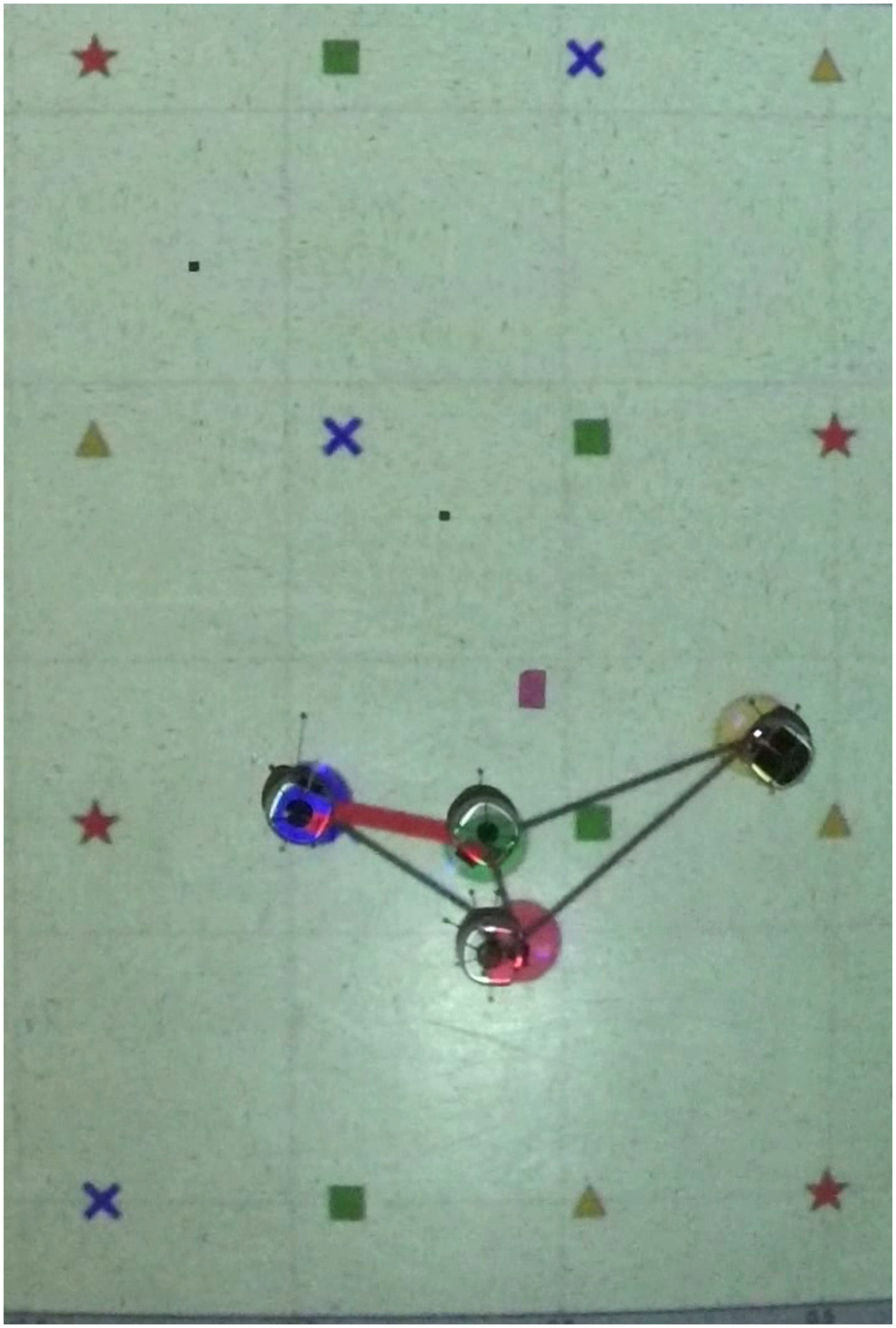}}
  \caption{Agents at 25.0s}
  \label{fig:t2ec}
\end{subfigure}
\begin{subfigure}{.23\textwidth}
  \centering
  \resizebox{1.8in}{!}{\includegraphics{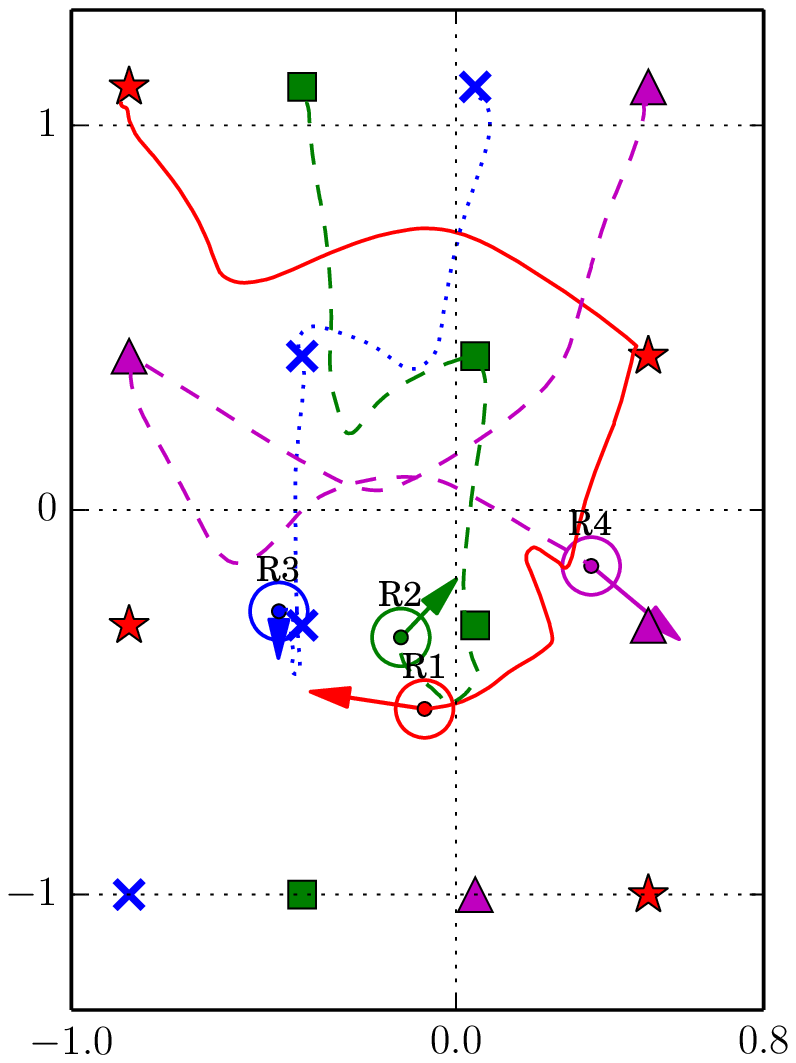}}
  \label{fig:t2pc}
\end{subfigure}
\\
\begin{subfigure}{.21\textwidth}
  \centering
  \resizebox{1.5in}{!}{\includegraphics{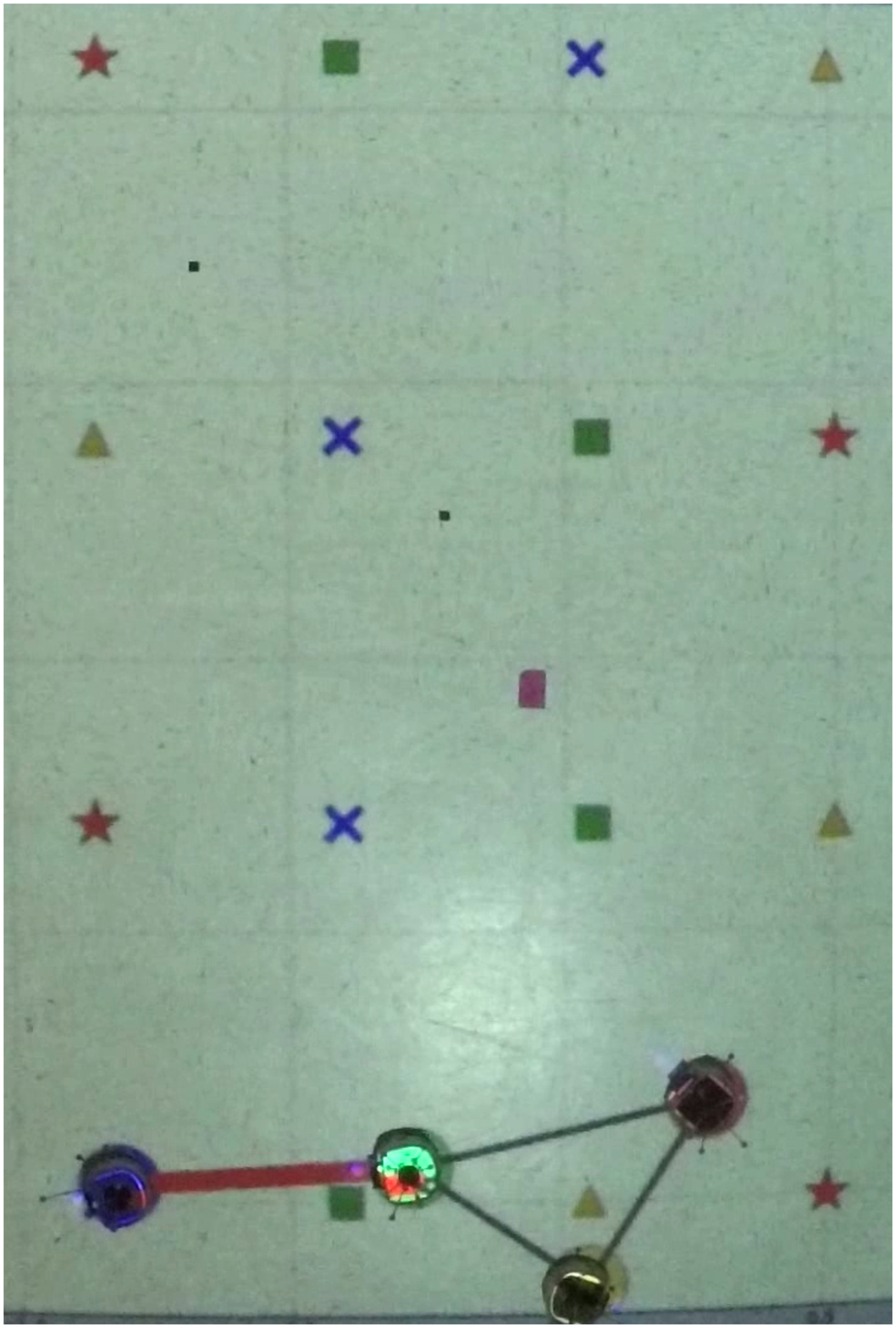}}
  \caption{Agents at 42.5s}
  \label{fig:t3ec}
\end{subfigure}%
\begin{subfigure}{.23\textwidth}
  \centering
  \resizebox{1.8in}{!}{\includegraphics{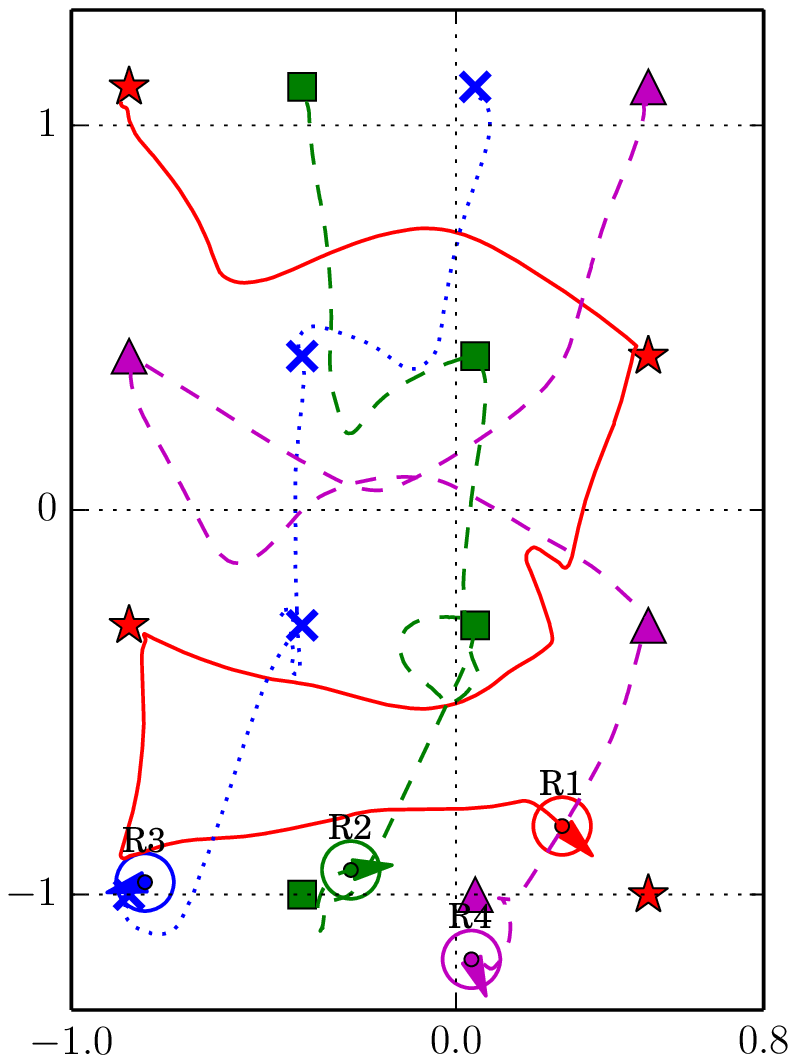}}
  \label{fig:t3pc}
\end{subfigure}%
\caption{Experiment of four mobile robots executing waypoint controllers regulated by safety and connectivity barrier certificates. The safety and connectivity distances are $D_s=0.15m$ and $D_c=0.6m$. The lines representing inter-robot connectivity are projected onto the ground using a projector.}
\label{fig:expsc}
\end{figure}

%\section*{APPENDIX}

\addtolength{\textheight}{-12cm}   % This command serves to balance the column lengths

%%%%%%%%%%%%%%%%%%%%%%%%%%%%%%%%%%%%%%%%%%%%%%%%%%%%%%%%%%%%%%%%%%%%%%%%%%%%%%%%

\bibliographystyle{abbrv}
\bibliography{mybib}
\end{document}